\documentclass[10pt,twocolumn,letterpaper]{article}

\usepackage{cvpr}
\usepackage{times}
\usepackage{epsfig}
\usepackage{graphicx}
\usepackage{amsmath}
\usepackage{amssymb}
\usepackage{url}
\usepackage{times}
\usepackage{algorithm}
\usepackage{algorithmic}
\usepackage{bm}
\usepackage{multirow}
\usepackage{paralist}
\usepackage{amsthm}

\newtheorem{thm}{Theorem}
\newtheorem{lemma}{Lemma}

\def \x {\mathbf{x}}

\def \R {\mathbb{R}}
\def \Ah {\widehat{A}}
\def \Mh {\widehat{M}}
\def \N {\mathcal{N}}
\def \Gh {\widehat{G}}
\def \Ocal {\mathcal{O}}
\def \D {\mathcal{D}}
\def \N {\mathcal{N}}
\def \u {\mathbf{u}}

\def \E {\mathrm{E}}

\usepackage[breaklinks=true,bookmarks=false]{hyperref}

\cvprfinalcopy % *** Uncomment this line for the final submission

\def\cvprPaperID{1640} % *** Enter the CVPR Paper ID here

\begin{document}

%%%%%%%%% TITLE
\title{Fine-Grained Visual Categorization via \\Multi-stage Metric Learning}

\author{Qi Qian$^1$\quad Rong Jin$^1$\quad Shenghuo Zhu$^2$\quad Yuanqing Lin$^3$\\
$^1$Department of Computer Science and Engineering\\
Michigan State University, East Lansing, MI, 48824, USA\\
$^2$Alibaba Group, Seattle, WA, 98101, USA\\
$^3$NEC Laboratories America, Cupertino, CA, 95014, USA\\
{\tt\small \{qianqi, rongjin\}@cse.msu.edu, shenghuo.zhu@alibaba-inc.com, ylin@nec-labs.com}
}

\maketitle
%\thispagestyle{empty}

%%%%%%%%% ABSTRACT
\begin{abstract}

Fine-grained visual categorization (FGVC) is to categorize objects into subordinate classes instead of basic classes. One major challenge in FGVC is the co-occurrence of two issues: 1) many subordinate classes are highly correlated and are difficult to distinguish, and 2) there exists the large intra-class variation (e.g., due to object pose). This paper proposes to explicitly address the above two issues via distance metric learning (DML). DML addresses the first issue by learning an embedding so that data points from the same class will be pulled together while those from different classes should be pushed apart from each other; and it addresses the second issue by allowing the flexibility that only a portion of the neighbors (not all data points) from the same class need to be pulled together. However, feature representation of an image is often high dimensional, and DML is known to have difficulty in dealing with high dimensional feature vectors since it would require $\Ocal(d^2)$ for storage and $\Ocal(d^3)$ for optimization. To this end, we proposed a multi-stage metric learning framework that divides the large-scale high dimensional learning problem to a series of simple subproblems, achieving $\Ocal(d)$ computational complexity. The empirical study with FVGC benchmark datasets verifies that our method is both effective and efficient compared to the state-of-the-art FGVC approaches.

\end{abstract}

%%%%%%%%% BODY TEXT
\section{Introduction}

%--------------------------------------
% Introduction to FGVC
%--------------------------------------

Fine-grained visual categorization (FGVC) aims to distinguish objects in subordinate classes. For example, dog images are classified into different breeds of dogs, such as ``Chihuahua'', ``Pug'', ``Samoyed'' and so on~\cite{khosla2011novel,ParkhiVZJ12}. One challenge of FGVC is that it has to handle the co-occurrence of two somewhat contradictory requirements: 1) it needs to distinguish
\begin{figure}[!ht]
\centering
\includegraphics[width = 3in,height = 2.2in]{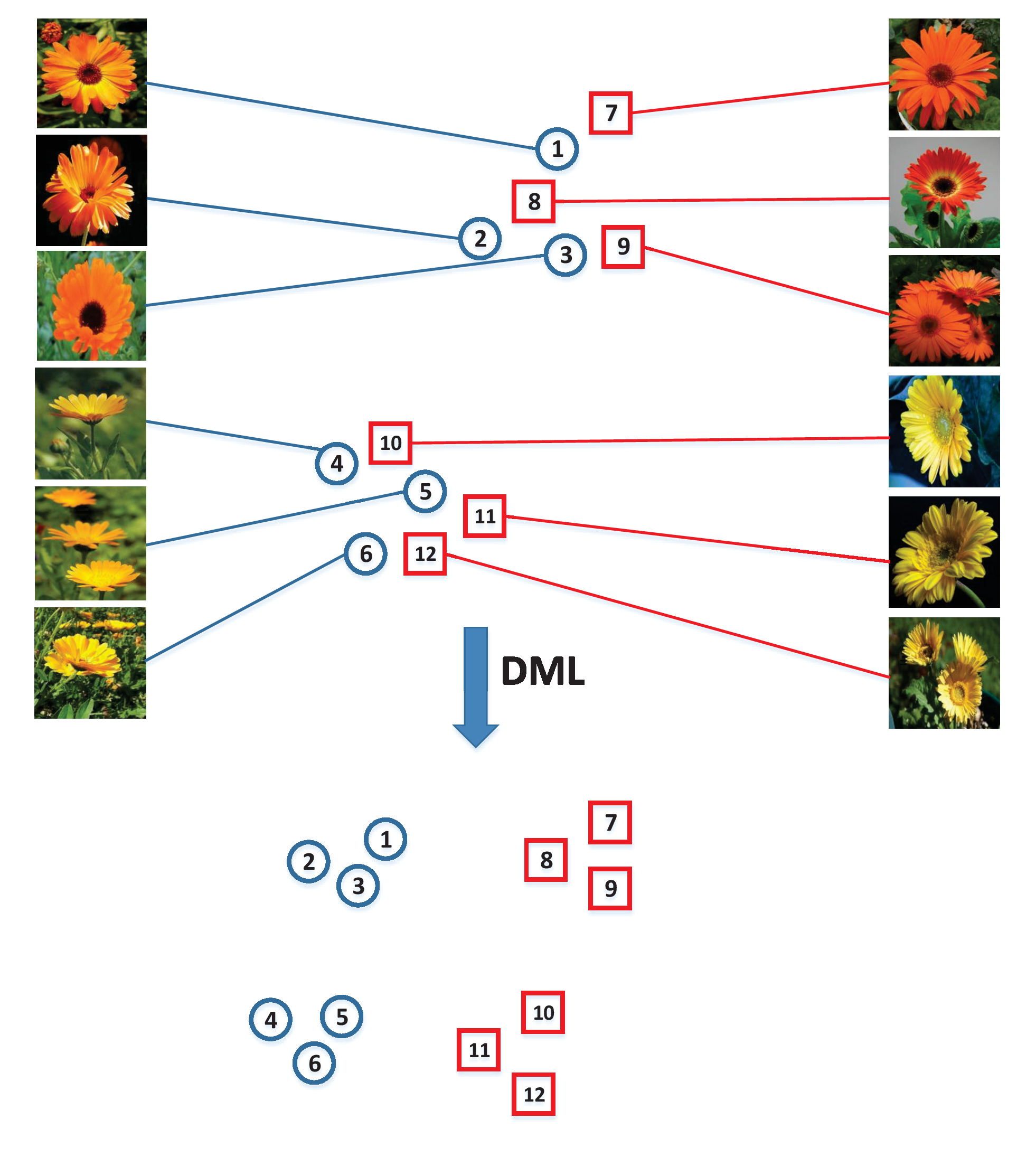}
\caption{Illustration of how DML learns the embedding that pulls together the data points from the same class and pushes apart the data points from different classes. Blue points are from the class ``English marigold'' while red ones are ``Barberton daisy''. An important note here is that our approach does not require to collapse all data points from each class to a single cluster and this allows the flexibility to model the intra-class variation.}\label{fig:1}
\end{figure}
many similar classes (e.g., the dog breeds that only have subtle differences), and 2) it needs to deal with the large intra-class variation (e.g., caused by different poses, examples, etc.).

%--------------------------------------
% Pipelines for FGVC
%--------------------------------------

The popular pipeline for FVGC consists of two steps, feature extraction step and classification step. The feature extraction step, which sometimes combines with segmentation~\cite{anelia13,ChaiRLGZ12,ParkhiVZJ12}, part localization~\cite{berg13,YangBWS12} or both~\cite{ChaiLZ13}, is to extract image level representations, and popular choices include LLC features~\cite{anelia13}, Fisher vectors~\cite{GavvesFSST13}, etc. A recent development is to train the convolutional neural network (CNN)~\cite{KrizhevskySH12} on a large-scale image dataset (e.g., ImageNet~\cite{ILSVRCarxiv14}) and then use the trained model to extract features~\cite{DonahueJVHZTD14}. The so-called deep learning features have demonstrated the state-of-the-art performance on FGVC datasets~\cite{DonahueJVHZTD14}. Note that there has been some difficulties in training CNN directly on FGVC datasets because the existing FGVC benchmarks are often too small~\cite{DonahueJVHZTD14} (only several tens of thousands of training images or less). In this paper, we simply take the state-of-the-art deep learning features without any other operators (e.g., segmentation) and focus on studying better classification approach to address the aforementioned two co-occurring requirements in FGVC.

For the classification step, many existing FGVC methods directly learn a single classifier for each fine-grained class using the one-vs-all strategy~\cite{anelia13,berg13,ChaiRLGZ12,YangBWS12}. Apparently, this strategy does not scale well to the number of fine-grained classes while the number of subordinate classes in FGVC could be very large (e.g., 200 classes in {\it birds11} dataset). Additionally, such one-vs-all scheme is only to address the first issue in the two issues, namely, it makes efforts to separate different classes without modeling intra-class variation. In this paper, we proposes a distance metric learning (DML) approach, aiming to explicitly handle the two co-occurring requirements with a {\it single} metric. Fig.~\ref{fig:1} illustrates how DML works for FGVC. It learns a distance metric that pulls neighboring data points of the same class close to each other and pushes data points from different classes far apart. By varying the neighborhood size when learning the metric, it is able to effectively handle the tradeoff between the inter-class and intra-class variation. With a learned metric, a $k$-nearest neighbor classifier will be applied to find the class assignment for a test image.

%--------------------------------------
% Introduction to DML
%--------------------------------------

Although numerous algorithms have been developed for DML~\cite{ChechikSSB10,DavisKJSD07,weinberger2009,XingNJR02}, most of them are limited to low dimensional data (i.e. no more than a few hundred dimensions) while the dimensionality of image data representation is usually higher than $10,000$~\cite{anelia13}. A straightforward approach toward high dimensional DML is to reduce the dimensionality of data by using the methods such as principle component analysis (PCA)~\cite{weinberger2009} and random projection~\cite{TsagkatakisS10}. The main problem with most dimensionality reduction methods is that they are unable to take into account the supervised information, and as a result, the subspaces identified by the dimensionality reduction methods are usually suboptimal.

%-------------------------------
% Proposed approaches
%-------------------------------

There are three challenges in learning a metric directly from the original high dimensional space:
\begin{compactitem}
\item {\it Large number of constraints}: A large number of training constraints are usually required to avoid the overfitting of high dimensional DML. The total number of triplet constraints could be up to $\Ocal(n^3)$ where $n$ is the number of examples.
\item {\it Computational challenge}: DML has to learn a matrix of size $d\times d$, where $d$ is the dimensionality of data and $d=134,016$ in our study. The $\Ocal(d^2)$ number of variables leads to two computational challenges in finding the optimal metric. First, it results in a slower convergence rate in solving the related optimization problem~\cite{QiTZCZ09}. Second, to ensure the learned metric to be positive semi-definitive (PSD), most DML algorithms require, at {\it every} iteration of optimization, projecting the intermediate solution onto a PSD cone, an expensive operation with complexity of $\Ocal(d^3)$ (at least $\Ocal(d^2)$).
\item {\it Storage limitation}: It can be expensive to simply save $\Ocal(d^2)$ number of variables in memory. For example, in our study, it would take more than 130 GB to store the completed metric in memory, which adds more complexity to the already difficult optimization problem.
\end{compactitem}

In this work, we propose a multi-stage metric learning framework for high dimensional DML that explicitly addresses these challenges. First, to deal with a large number of constraints used by high dimensional DML, we divide the original optimization problem into multiple stages. At each stage, only a small subset of constraints that are difficult to be classified by the currently learned metric will be adaptively sampled and used to improve the learned metric. By setting the regularizer appropriately, we can prove that the final solution is optimized over all appeared constraints. Second, to handle the computational challenge in each subproblem, we extend the theory of {\it dual random projection}~\cite{Zhang13}, which was originally developed for linear classification problems, to DML. The proposed method enjoys the efficiency of random projection, and on the other hand learns a distance metric of size $d\times d$. This is in contrast to most dimensionality reduction methods that learn a metric in a {\it reduced} space. Finally, to handle the storage problem, we propose to maintain a low rank copy of the learned metric by a randomized algorithm for low rank matrix approximation. It not only accelerates the whole learning process but also regularizes the learned metric to avoid overfitting. Extensive comparisons on benchmark FGVC datasets verify the effectiveness and efficiency of the proposed method.

%----------organization
The rest of the paper is organized as follows: Section~\ref{sec:relate} summarizes related work for DML. Section~\ref{sec:method} describes the details of the proposed method. Section~\ref{sec:exp} shows the results of the empirical study, and Section~\ref{sec:conclude} concludes this work with future directions.

\section{Related Work}
\label{sec:relate}

%-----------------
% DML
%-----------------
Many algorithms have been developed for DML~\cite{DavisKJSD07,weinberger2009,XingNJR02} and a detailed review can be found in two survey papers~\cite{Kulis13,liu2006}. Some of them are based on pairwise constraints~\cite{DavisKJSD07,XingNJR02}, while others focus on optimizing triplet constraints~\cite{ChechikSSB10,weinberger2009}. In this paper, we adopt triplet constraints, which exactly serve our purpose for addressing the second issue of FGVC. Although numerous studies were devoted to DML, few examined the challenges of high dimensional DML. A common approach for high dimensional DML is to project data into a low dimensional space, and learn a metric in the space of reduced dimension, which often leads to a suboptimal performance. An alternative approach is to assume $M$ to be of low rank by writing $M$ as $M = LL^{\top}$~\cite{DavisD08,weinberger2009}, where $L$ is a tall rectangle matrix and the rank of $M$ is fixed in advance of applying DML methods. Instead of learning $M$, these approaches directly learn $L$ from data. The main shortcoming of this approach is that it has to solve a non-convex optimization problem, making it computationally less attractive. Several recent studies~\cite{Lim13,QiTZCZ09} address high dimensional DML by assuming $M$ to be sparse. Although resolving the storage problem, they still suffer from high cost in optimizing $\Ocal(d^2)$ variables.

\section{Multi-stage Metric Learning}
\label{sec:method}

The proposed DML algorithm focuses on triplet constraints so as to pull the small portion of nearest examples from the same class together~\cite{weinberger2009}. Let $\D = \{(\x_i, y_i), i=1, \ldots, n\}$ be a collection of $n$ training images, where $\x_i \in \R^d$ and $y_i$ is the class assignment of $\x_i$. Given a distance metric $M$, the distance between two data points $\x_i$ and $\x_j$ is measured by
\[
d_M(\x_i, \x_j) = (\x_i - \x_j)^{\top}M(\x_i - \x_j)
\]
Let $\{\x_i^t, \x_j^t, \x_k^t\}(t=1, \ldots, N)$ be a set of $N$ triplet constraints derived from the training examples in $\D$. Since in each constraint $(\x_i^t, \x_j^t, \x_k^t)$, $\x_i^t$ and $\x_j^t$ share the same class assignment which is different from that of $\x_k^t$, we expect $d_M(\x_i^t, \x_j^t) < d_M(\x_i^t, \x_k^t)$. As a result, the optimal distance metric $M$ is learned by solving the following optimization problem
\begin{eqnarray}
\min\limits_{M \in S_d, M\succeq 0}\! \frac{\lambda}{2}\|M\|_F^2\! + \!\sum_{t=1}^N\! \ell(d_M\!(\x_i^t,\! \x_k^t)\! -\! d_M\!(\x_i^t,\! \x_j^t)) \label{eqn:1}
\end{eqnarray}
where $S_d$ includes all $d\times d$ real symmetric matrices and $\ell(\cdot)$ is a loss function that penalizes the objective function when $d_M(\x_i^t, \x_k^t)$ is not significantly larger than $d_M(\x_i^t, \x_j^t)$. In this study, we choose the smoothed hinge loss~\cite{tzhang2012} that appears to be more effective for optimization than the hinge loss while keeping the benefit of large margin
\begin{eqnarray*}
\ell(x) &=& \left\{\begin{array}{l@{\quad:\quad}r} 0 & x>1\\ 1-x-\gamma/2 & x<1-\gamma\\ \frac{1}{2\gamma} (1-x)^2 & o.w. \end{array}\right.
\end{eqnarray*}

One main computational challenge of DML comes from the PSD constraint $M \succeq 0$ in (\ref{eqn:1}). We address this challenge by following the one projection paradigm~\cite{ChechikSSB10} that first learns a metric $M$ without the PSD constraint and then projects $M$ to the PSD cone at the very end of the learning process. Hence, in this study, we will focus on the following optimization problem for FGVC
\begin{eqnarray}
\min\limits_{M \in S_d} \frac{\lambda}{2}\|M\|_F^2 + \sum_{t=1}^N \ell(\langle A_t, M \rangle) \label{eqn:opt}
\end{eqnarray}
where $A_t = (\x_i^t - \x_k^t)(\x_i^t - \x_k^t)^{\top} - (\x_i^t - \x_j^t)(\x_i^t - \x_j^t)^{\top}$ is introduced as a matrix representation for each triplet constraint, and $\langle \cdot, \cdot \rangle$ represents the dot product between two matrices.

We will discuss the strategies to address the three challenges of high dimensional DML, and summarize the framework of high dimensional DML for FGVC at the end of this section.

\subsection{Constraints Challenge: Multi-stage Division}

In order to reliably determine the distance metric in a high dimensional space, a large number of training examples are needed to avoid the overfitting problem. Since the number of triplet constraints can be $\Ocal(n^3)$, the number of summation terms in (\ref{eqn:opt}) can be extremely large, making it difficult to effectively solve the optimization problem in (\ref{eqn:opt}). Although learning with active set may help reduce the number of constraints~\cite{weinberger2009}, the number of active constraints can still be very large since many images in FGVC from different categories are visually similar, leading to many mistakes. To address this challenge, we divide the learning process into multiple stages. At the $s$-th stage, let $M_{s-1}$ be the distance metric learned from the last stage. We sample a subset of active triplet constraints that are difficult to be classified by $M_{s-1}$ (i.e., incur large hinge loss){\footnote{The strategy of finding hard constraints at each stage is also applied by cutting plane methods~\cite{nesterov2004} and active learning~\cite{settles2010}.}}. Given $M_{s-1}$ and the sampled triplet constraints $\N_s$, we update the distance metric by solving the following optimization problem
\begin{eqnarray}\label{primal:s}
\min_{M_s\in S_d} \frac{\lambda}{2}\|M_s-M_{s-1}\|_F^2+\sum_{t \in \N_s} \ell(\langle A_t,M_s \rangle)
\end{eqnarray}

Although only a small size of constraints is used to improve the metric at each stage, we have
\begin{thm}
The metric learned by solving the problem (\ref{primal:s}) also optimizes the following objective function
\[
\min_{M\in S_d}\frac{\lambda}{2}\|M\|_F^2 + \sum_{k=1}^{s} \sum_{t \in \N_k} \ell(\langle A_t, M\rangle )
\]
\end{thm}

\begin{proof}
Consider the objective function for the first $s$ stages
\begin{eqnarray}
\min_{M\in S_d}\underbrace{\frac{\lambda}{2}\|M\|_F^2 + \sum_{k=1}^{s-1} \sum_{t \in \N_k} \ell(\langle A_t, M\rangle )}_{:= \mathcal{L}_{s-1}(M)} + \sum_{t \in \N_s} \ell(\langle A_t, M\rangle) \label{eqn:obj-s}
\end{eqnarray}
It is obvious that $\mathcal{L}_{s-1}$ is strongly convex, so we have (Chapter 9, \cite{boyd2009convex})
\begin{eqnarray*}
\mathcal{L}_{s-1}(M)\! =\! \mathcal{L}_{s-1}(M_{s-1})\! +\! \langle \nabla\mathcal{L}_{s-1}(M_{s-1}),M-M_{s-1} \rangle\\
+\frac{1}{2} \langle (M-M_{s-1})\nabla^2\mathcal{L}_{s-1}(M'),M-M_{s-1} \rangle
\end{eqnarray*}
for some $M'$ between $M$ and $M_{s-1}$.

Since $M_{s-1}$, the solution obtained from the first $s-1$ stages, approximately optimizes $\mathcal{L}_{s-1}(M)$ and $\mathcal{L}_{s-1}$ is $\lambda$-strongly convex, then
\begin{eqnarray}
\mathcal{L}_{s-1}(M) \approx \mathcal{L}_{s-1}(M_{s-1}) + \frac{\lambda}{2} \|M - M_{s-1}\|_F^2 \label{eqn:approx}
\end{eqnarray}

We finish the proof by replacing $\mathcal{L}_{s-1}(M)$ in (\ref{eqn:obj-s}) with the approximation in (\ref{eqn:approx}).
\end{proof}

\paragraph{Remark} This theorem demonstrates that the metric learned from the last stage is optimized over constraints from all stages. Therefore, the original problem could be divided into several subproblems and each of them has an affordable number of active constraints. Fig.~\ref{fig:3} summaries the framework of the multi-stage learning procedure.

\begin{figure}[!ht]
\centering
\includegraphics[width = 3in,height = 1.2in]{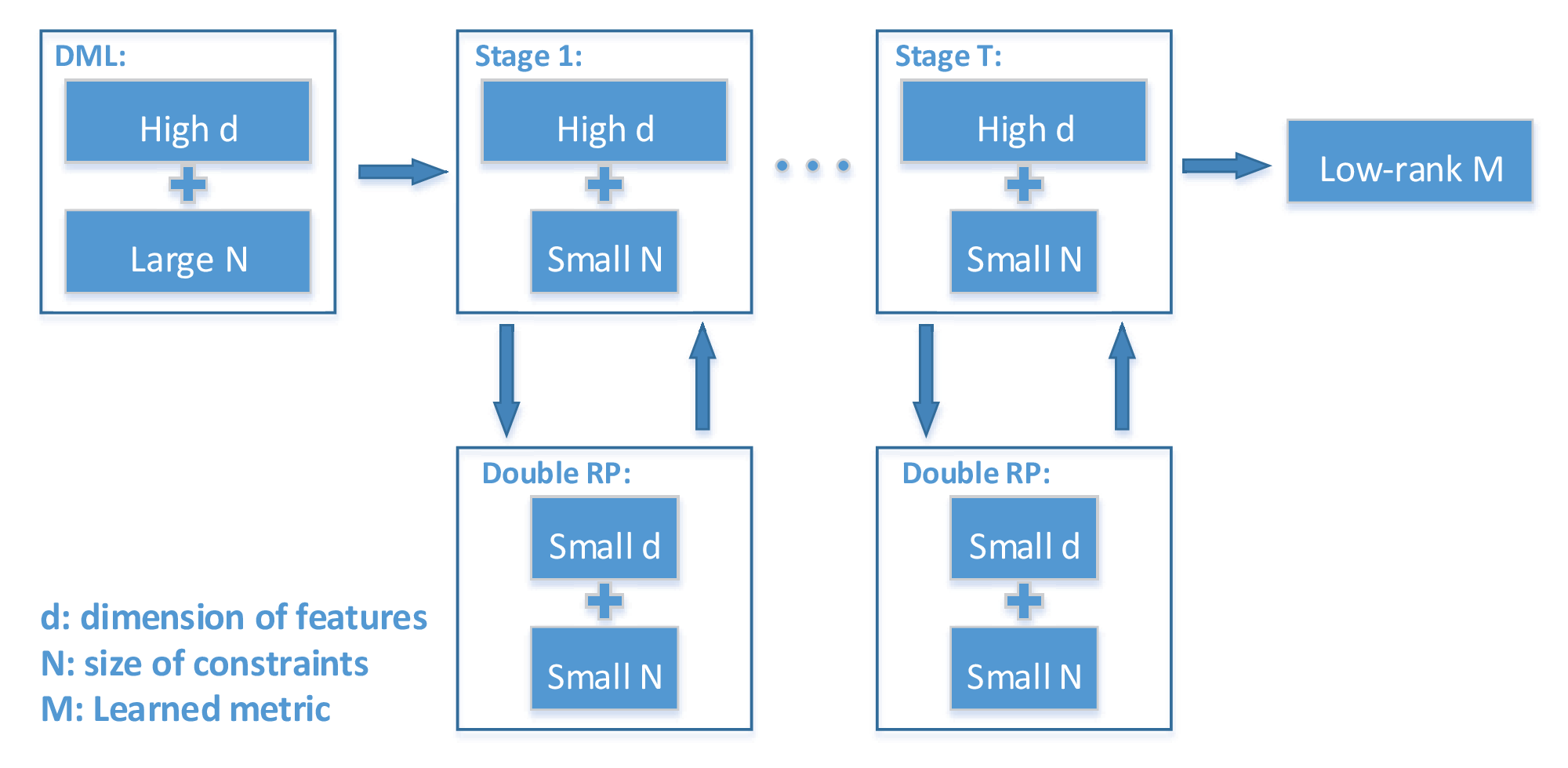}\\
\caption{The framework of the proposed method.}\label{fig:3}
\end{figure}

\subsection{Computational Challenge: Dual Random Projection}

Now we try to solve the high dimensional subproblem by dual random projection technique. To simplify the analysis, we investigate the subproblem at the first stage and the following stages could be analyzed in the same way. By introducing the convex conjugate $\ell_*$ for $\ell$ in (\ref{primal:s}), the dual problem of DML is
\begin{eqnarray}\label{dual:1}
 \max\limits_{{\bm \alpha} \in \R^{|\N_1|}} -\sum_{t=1}^{|\N_1|} \ell_*(\alpha_t) - \frac{1}{2\lambda}{\bm \alpha}^\top G {\bm \alpha}
\end{eqnarray}
where $\alpha_t$ is the dual variable for $A_t$ and $G$ is a matrix defined as $G_{a,b} = \langle A_a,A_b\rangle$. $M_1 = - \frac{1}{\lambda} \sum_{t=1}^{|\N_1|} \alpha_t A_t$ by setting the gradient with respect to $M_1$ to zero.
Let $R_1, R_2\in \R^{d\times m}$ be two Gaussian random matrices, where $m$ is the number of random projections ($m \ll d$) and $R_1^{i,j}, R_2^{i,j} \sim \N(0, 1/m)$. For each triplet constraint, we project its representation $A_t$ into the low dimensional space using the random matrices, i.e. $\Ah_t = R_1^{\top}A_tR_2$. By using double random projections, which is different from the single random projection in~\cite{Zhang13}, we have
\begin{lemma}
$\forall A_a, A_b$, the double random projections preserve the pairwise similarity between them: $\E[\langle \Ah_a,\Ah_b\rangle] = \langle A_a, A_b\rangle$
\end{lemma}
The proof is straightforward. According to the lemma, the dual variables in (\ref{dual:1}) can be estimated in the low dimensional space as
\begin{eqnarray}\label{dual:2}
 \max\limits_{{\bm \hat{\alpha}} \in \R^{|\N_1|}} -\sum_{t=1}^{|\N_1|} \ell_*(\hat{\alpha}_t) - \frac{1}{2\lambda}{\bm \hat{\alpha}}^\top \Gh {\bm \hat{\alpha}}
\end{eqnarray}
where $\Gh(a,b) = \langle \Ah_a,\Ah_b\rangle$.
Then, by the definition of convex conjugate, each dual variable $\hat{\alpha}_t$ in (\ref{dual:2}) can be further estimated by $\ell'(\langle \Ah_t, \Mh_1 \rangle)$, where $\Mh_1 \in \R^{m\times m}$ is the metric learned in the reduced space. Generally, $\Mh_s$ is learned by solving the following optimization problem
\begin{eqnarray}\label{primal:r}
\min_{\Mh_s\in S_m} \frac{\lambda}{2}\|\Mh_s-\Mh_{s-1}\|_F^2+\sum_{t=1}^{|\N_s|}\ell(\langle\Ah_t,\Mh_s\rangle)
\end{eqnarray}
Since the size of $\Mh_s \in \R^{m\times m}$ is significantly smaller than that of $M_s$, (\ref{primal:r}) can be solved much more efficiently than (\ref{primal:s}). In our implementation, a simple stochastic gradient descent (SGD) method is developed to efficiently solve the optimization problem in (\ref{primal:r}). Given $\Mh_1$, the final distance metric $M_1 \in \R^{d\times d}$ in the original space is estimated as
\begin{eqnarray}\label{eqn:M}
M_1 &=& - \frac{1}{\lambda} \sum_{t=1}^{|\N_1|} \hat{\alpha}_t A_t
\end{eqnarray}
%According to~\cite{Zhang13}, with a high probability, we have $\|M - M_*\|_F \leq O(1/\sqrt{m})\|M_*\|_F$ under favored condition, where $M_*$ is the optimal solution to (\ref{primal:s}).

\subsection{Storage Challenge: Low Rank Approximation}
\label{randomized}

Although (\ref{eqn:M}) allows us to recover the distance metric $M$ in original $d$ dimensional space from the dual variables $\{\alpha_t\}_{t=1}^{|\N|}$, it is expensive, if not impossible, to save $M$ in memory since $d$ is very large in FGVC~\cite{anelia13}. To address this challenge, instead of saving $M$, we propose to save the low rank approximation of $M$. More specifically, let $\sigma_1, \ldots, \sigma_r$ be the first $r \ll d$ eigenvalues of $M$, and let $\u_1, \ldots, \u_r$ be the corresponding eigenvectors. We approximate $M$ by a low rank matrix $M' = \sum_{i=1}^r \sigma_i\u_i\u_i^{\top} = LL^\top$. Different from existing DML methods that directly optimize $L$~\cite{WangYYLHG10}, we obtain $M$ first and then decompose it to avoid suboptimal solution. Unlike $M$ that requires $\Ocal(d^2)$ storage space, it only takes $\Ocal(rd)$ space to save $M'$ and $r$ could be an arbitrary value. In addition, the low rank metric accelerates the sampling step by reducing the cost of computing distance from $\Ocal(d)$ to $\Ocal(r)$. Low rank is also a popular regularizer to avoid overfitting when learning high dimensional metric~\cite{Lim13}. However, the key issue is how to efficiently compute the eigenvectors and eigenvalues of $M$ at {\it each} stage. This is particularly challenging in our case as $M$ in (\ref{eqn:M}) even can not be computed explicitly due to its large size.

To address this problem, first we investigate the structure of the recovering step for the $s$-th stage as in (\ref{eqn:M})
\begin{eqnarray*}
M_s &=& M_{s-1}- \frac{1}{\lambda} \sum_{t=1}^{|\N_s|} \alpha_t^s A_t^s\\
&=& M_{s-2} - \frac{1}{\lambda} (\sum_{t=1}^{|\N_s|} \alpha_t^s A_t^s+\sum_{t=1}^{|\N_{s-1}|} \alpha_t^{s-1} A_t^{s-1})\\
&=& - \frac{1}{\lambda}\sum_{k=1}^s \sum_{t=1}^{|\N_k|} \alpha_t^k A_t^k
\end{eqnarray*}

Therefore, we can express the summation as matrix multiplications. In particular, for each triplet $(\x_i^t,\x_j^t,\x_k^t)$, we denote its dual variable by $\alpha = \ell'(\langle \Ah, \Mh\rangle)$ and set the corresponding entries in a sparse matrix $C$ as
\begin{eqnarray}
C(i,j) = \frac{\alpha}{\lambda} ,\ C(j,i) = \frac{\alpha}{\lambda} ,\ C(j,j) = -\frac{\alpha}{\lambda} & &\nonumber\\
C(i,k) = -\frac{\alpha}{\lambda} ,\ C(k,i) = -\frac{\alpha}{\lambda} ,\ C(k,k) = \frac{\alpha}{\lambda} & &\label{eqn:C}
\end{eqnarray}
It is easy to verify that $M$ can be written as
\begin{eqnarray}\label{eqn:3}
M = XCX^\top
\end{eqnarray}

Second, we exploit the randomized theory~\cite{2009arXiv} to efficiently compute the eigen-decomposition of $M$. More specifically, let $R \in \R^{d\times q}$ ($q\ll d$) be an Gaussian random matrix. According to~\cite{2009arXiv}, with an overwhelming probability, most of the top $r$ eigenvectors of $M$ lie in the subspace spanned by the column vectors in $MR$ provided $q \geq r + k$, where $k$ is a constant independent from $d$. The limitation of the method is that it requires the appearance of the matrix $M$ for computing $MR$ while keeping the whole matrix is unaffordable here. Fortunately, by replacing $M$ with $XCX^\top$ according to (\ref{eqn:3}), we can approximate the top eigenvectors of $M$ by those of $XCX^\top R$ that is of size $d\times q$ and can be computed efficiently since $C$ is a sparse matrix. The overall computational cost of the proposed algorithm for low rank approximation is only $\Ocal(qnd)$, which is linear in $d$.  Note that the sparse matrix $C$ is cumulated over all stages.

Alg.~\ref{alg:1} summarizes the key steps of the proposed approach for low rank approximation, where $qr$ and $eig$ stand for QR and eigen decomposition of a matrix. Note that the distributed computing is particularly effective for the realization of the algorithm because the matrix multiplications $XCX^{\top}R$ can be accomplished in parallel, which is helpful when $n$ is also large.

\begin{algorithm}[h]
\caption{An Efficient Algorithm for Recovering $M$ and Projecting It onto PSD Cone from $\Mh$}
\begin{algorithmic}[1]

\STATE {\bf Input:} Dataset $X\in \R^{d\times n}$, $\Mh \in \R^{m\times m}$, the number of random combinations $q$

\STATE Compute a Gaussian random matrix $R \in \R^{d\times q}$
\STATE Compute the sparse matrix $C$ using (\ref{eqn:C})
\STATE $Y = R\times X^\top ,\ Y = Y\times C ,\ Y = Y\times X$
\STATE $[Q,R] = qr(Y)$
\STATE $B = Q^\top\times X^\top ,\ B = B\times C ,\ B = B\times X$
\STATE $[U,\Sigma] = eig(B)$
\STATE $U = Q*U$
\RETURN $L = [\sqrt{\sigma_1}\u_1,\cdots,\sqrt{\sigma_r}\u_r]$ and $M = LL^\top$, where $\u_i$ is the $i$th column of $U$ and $\sigma_i$ is the $i$th positive diagonal element of $\Sigma$
\end{algorithmic}\label{alg:1}
\end{algorithm}

Alg.~\ref{alg:2} shows the whole picture of the proposed method.

\begin{algorithm}[t]
\caption{The {\bf M}ulti-{\bf s}tage {\bf M}etric {\bf L}earning Framework for High Dimensional DML (MsML)}
\begin{algorithmic}[1]

\STATE {\bf Input:} Dataset $X\in \R^{d\times n}$, the number of random projections $m$, the number of random combinations $q$, and the number of stages $T$

\STATE Compute two Gaussian random matrices $R_1, R_2 \in \R^{d\times m}$
\STATE Initialize $\Mh_0 = \mathbf{0} \in \R^{m\times m}$ and $M_0 = \mathbf{0} \in \R^{d\times d}$
\FOR{$s = 1, \ldots, T$}
    \STATE Sample one epoch active triplet constraints using $M_{s-1}$
    \STATE Estimate $\Mh_s$ by solving the optimization problem as in (\ref{primal:r}) with SGD
    \STATE Recover the distance metric $M_s$ in the $d$ dimensional space using Alg.~\ref{alg:1}
\ENDFOR
\RETURN $M_T$
\end{algorithmic}\label{alg:2}
\end{algorithm}

\section{Experiments}
\label{sec:exp}
%\subsection{Setting}

DeCAF features~\cite{DonahueJVHZTD14} are extracted as the image representations in the experiments. Although it is from the activation of a deep convolutional network, which is trained on ImageNet~\cite{KrizhevskySH12}, it outperforms conventional visual features on many general tasks~\cite{DonahueJVHZTD14}. We concatenate features from the last three fully connected layers (i.e., DeCAF$_{5+6+7}$) and the dimension of resulting features is $51,456$.

We apply the proposed algorithm to learn a distance metric and use the learned metric together with a smoothed $k$-nearest neighbor classifier, a variant of $k$-NN, to predict the class assignments for test examples. Different from conventional $k$-NN, it first obtains $k$ reference centers for each class by clustering training images in each class into $k$ clusters. Then, it computes the query's distance to each class as the soft min of the distances between the test image and corresponding reference centers, and assigns the test image to the class with the shortest distance. It is more efficient when predicting, especially for large-scale training set, and the performance is similar to that of conventional one. We refer to the classification approach based on the metric learned by the proposed algorithm and the smoothed $k$-NN as {\bf MsML}, and the smoothed $k$-NN with Euclidean distance in the original space as {\bf Euclid}. Although the size of the covariance matrix is very large ($51,456\times 51,456$), its rank is low due to the small number of training examples, and thus PCA can be computed explicitly. The state-of-the-art DML algorithm, i.e. {\bf LMNN}~\cite{weinberger2009} with PCA as preprocess, is also included in comparison.
The one-vs-all strategy, based on the implementation of LIBLINEAR~\cite{REF08a}, is used as a baseline for FGVC, with the regularization parameter varied in the range $\{10^i\} (i=-2,\cdots,3)$. We refer to it as {\bf LSVM}. We also include the state-of-the-art results for FGVC in our evaluation. All the parameters used by MsML are set empirically, with the number of random projections $m = 100$ and the number of random combinations $q = 600$. PCA is applied for LMNN to reduce the dimensionality to $m$ before the metric is learned. LMNN is implemented by the code from the original authors and the recommended parameters are used~\footnote{We did vary the parameter slightly from the recommended values and did not find any noticeable change in the classification accuracy.}. To ensure that the baseline method fully exploits the training data, we set the maximum number of iterations for LMNN as $10^4$. These parameter values are used throughout all the experiments. All training/test splits are provided by datasets. {\bf Mean accuracy}, a standard evaluation metric for FGVC, is used to evaluate the classification performance. All experiments are run on a single machine with $16$ $2.10$GHz cores and $96$GB memory.

\subsection{Oxford Cats\&Dogs}
{\it cats\&dogs} contains $7,349$ images from $37$ cat and dog species~\cite{ParkhiVZJ12}. There are about 100 images per class for training and the rest are for test. Table~\ref{rlr} summaries the results. First, we observe that MsML is more accurate than the baseline LSVM. This is not surprising because the distance metric is learned from the training examples of all class assignments. This is in contrast to the one-vs-all approach used in LSVM that the classification function for a class $C$ is learned only by the examples with the class assignment of $C$. Second, our method performs significantly better than the baseline DML method, indicating that the unsupervised dimension reduction method PCA may result in suboptimal solutions for DML. Fig.~\ref{fig:4} compares the images that are most similar to the query images using the metric learned by the proposed algorithm (Column 8-10) to those based on the metric learned by LMNN (Column 5-7) and Euclid (Column 2-4). We observe that more images from the same class as the query are found by the metric learned by MsML than LMNN. For example, MsML is able to capture the difference between two cat species (longhair v.s. shorthair) while LMNN returns the very similar images with wrong class assignments. Third, MsML has overwhelming performance compared to all state-of-the-art FGVC approaches. Although the method~\cite{ParkhiVZJ12} using ground truth head bounding box and segmentation achieves $59.21\%$, MsML is $20\%$ better than it with only image information, which shows the advantage of the proposed method. Finally, it takes less than $0.2$ second to extract DeCAF features per image based on a $CPU$ implementation while a simple segmentation operator costs more than 2.5 seconds as reported in the study~\cite{anelia13}, making the proposed method for FGVC more appealing.

\begin{table}[!ht]
\centering
\caption{Comparison of mean accuracy($\%$) on {\it cats\&dogs} dataset. ``\#'' means that more information (e.g., ground true segmentation) is used by the method.}
\label{rlr}
\begin{tabular}{|l||c|}
\hline
Methods &    Mean Accuracy ($\%$)             \\\hline
Image only~\cite{ParkhiVZJ12}        &39.64  \\\hline
Det+Seg~\cite{anelia13}              &54.30   \\\hline
Image+Head+Body\#~\cite{ParkhiVZJ12}   &59.21   \\\hline
\hline
Euclid   &72.60                              \\\hline
LSVM     &77.63                            \\\hline
LMNN     &76.24                             \\\hline
MsML      &80.45                         \\\hline
MsML+     &81.18                             \\\hline
\end{tabular}
\end{table}

\begin{figure*}[!ht]
\centering
\includegraphics[width = 6in,height = 2in]{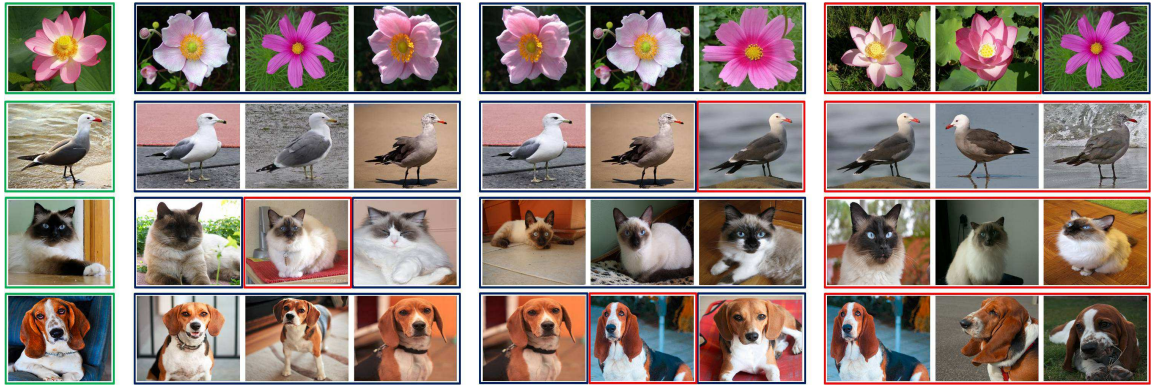}
\caption{Examples of retrieved images. The first column indicates the query images highlighted by green bounding boxes. Columns $2$-$4$ include the most similar images measured by Euclid. Columns $5$-$7$ show those by the metric from LMNN. Columns $8$-$10$ are from the metric of MsML. Images in columns $2$-$10$ are highlighted by red bounding boxes when they share the same category as queries, and blue bounding boxes if they are not.}\label{fig:4}
\end{figure*}

To evaluate the performance of MsML for extremely high dimensional features, we concatenate conventional features by using the pipeline for visual feature extraction that is outlined in~\cite{anelia13}. Specifically, we extract HOG~\cite{DalalT05} features at 4 different scales and encode them to $8K$ dimensional feature dictionary by the LLC method~\cite{WangYYLHG10}. A max pooling strategy is then used to aggregate local features into a single vector representation. Finally, $82,560$ features are extracted from each image and the total dimension is up to $134,016$. MsML with the combined features is denoted as {\bf MsML+} and it further improves the performance by about $1\%$ as in Table~\ref{rlr}. Note that the time of extracting these high dimensional conventional features is only $0.5$ second per image, which is still much cheaper than any segmentation or localization operator.

\subsection{Oxford 102 Flowers}

{\it 102flowers} is the Oxford flowers dataset for flower species~\cite{NilsbackZ08}, which consists of 8189 images from 102 classes. Each class has 20 images for training and rest for test. Table~\ref{rlr2} shows the results from different methods. We have the similar conclusion for the baseline methods. That is, MsML outperforms LSVM and LMNN significantly. Although LSVM already performs very well, MsML further improves the accuracy. Additionally, it is observed that even the performances of state-of-the-art methods with segmentation operators are much worse than that of MsML. Note that GT~\cite{NilsbackZ08} uses hand annotated segmentations followed by multiple kernel SVM, while MsML outperforms it about $3\%$ without any supervised information, which confirms the effectiveness of the proposed method.

\begin{table}[!ht]
\centering
\caption{Comparison of mean accuracy($\%$) on {\it 102flowers} dataset. ``\#'' means that more information (e.g., ground true segmentation) is used by the method.}
\label{rlr2}
\begin{tabular}{|l||c|}
\hline
Methods &    Mean Accuracy ($\%$)             \\\hline
Combined CoHoG~\cite{ItoK10} & 74.80         \\\hline
Combined Features~\cite{Nilsback09}  & 76.30 \\\hline
BiCoS-MT~\cite{ChaiLZ11}     & 80.00         \\\hline
Det+Seg~\cite{anelia13}  & 80.66   \\\hline
TriCoS~\cite{ChaiRLGZ12} & 85.20\\\hline
GT\#~\cite{NilsbackZ08}    & 85.60\\\hline
\hline
Euclid   &   76.21                           \\\hline
LSVM     &   87.14                          \\\hline
LMNN     &   81.93                           \\\hline
MsML     &   88.39                          \\\hline
MsML+    &   89.45                          \\\hline
\end{tabular}
\end{table}

Fig.~\ref{fig:2} illustrates the changing trend of test mean accuracy as the number of stages increases. We observe that MsML converges very fast, which verifies that multi-stage division is essential to the proposed framework.

\makeatletter\def\@captype{figure}\makeatother
\begin{minipage}{0.21\textwidth}
\centering
\includegraphics[width = 1.3in,height = 1in]{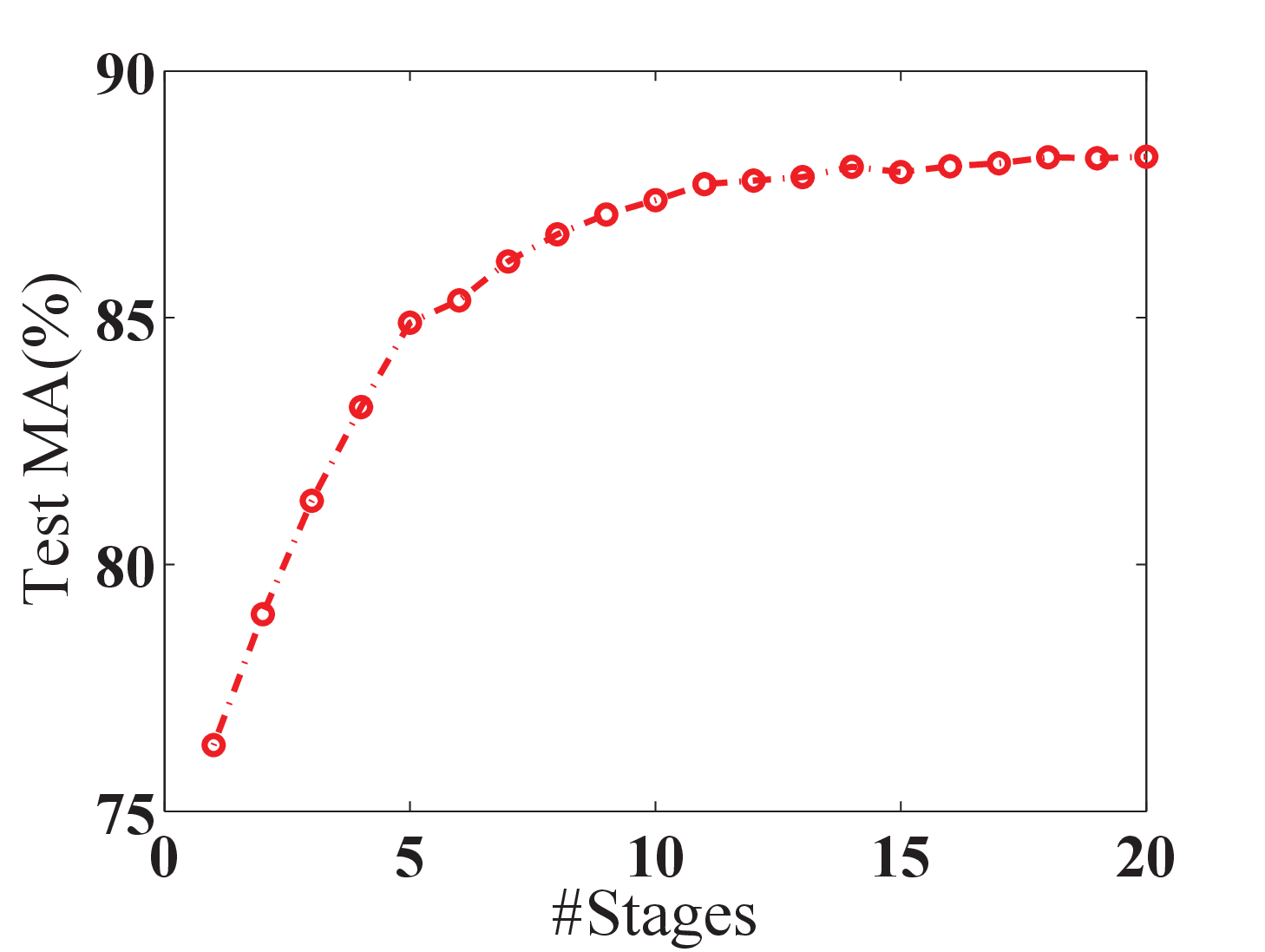}
\caption{Convergence curve of the proposed method on {\it 102flowers}.}\label{fig:2}
\end{minipage}
\makeatletter\def\@captype{figure}\makeatother
\begin{minipage}{0.21\textwidth}
\centering
\includegraphics[width = 1.3in,height = 1in]{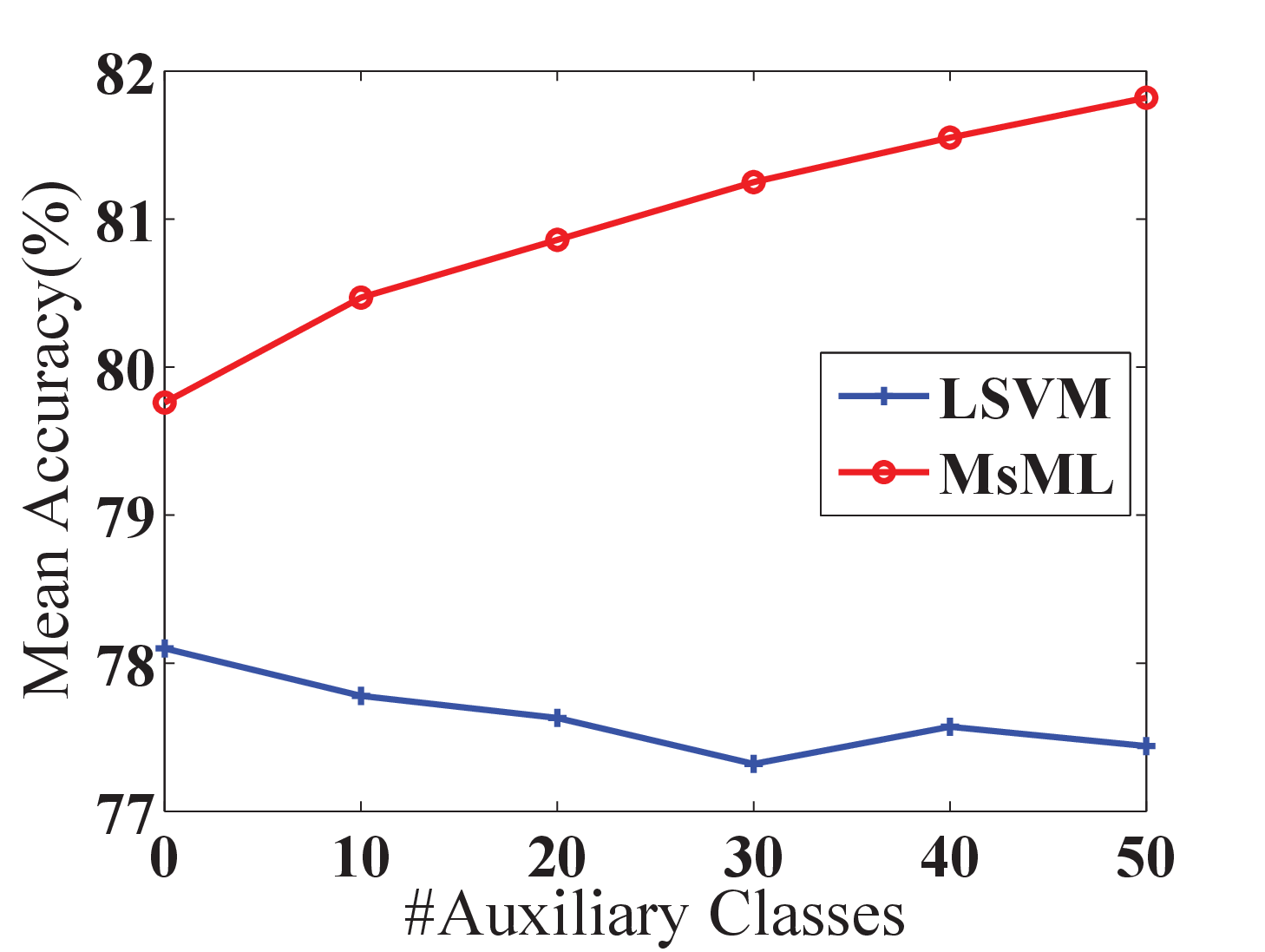}
\caption{Comparison with different size of classes on {\it birds11}.}\label{fig:5}
\end{minipage}

\begin{table}[!ht]
\centering
\caption{Comparison of mean accuracy($\%$) on {\it birds11} dataset. ``*'' denotes the method that mirrors training images.}
\label{rlr3}
\begin{tabular}{|l||c|}
\hline
Methods &    Mean Accuracy ($\%$) \\\hline
Symb~\cite{ChaiLZ13}  &56.60    \\\hline
POOF~\cite{BergB13}   &56.78   \\\hline
Symb*~\cite{ChaiLZ13}    & 59.40    \\\hline
Ali*~\cite{GavvesFSST13} & 62.70   \\\hline
DeCAF+DPD~\cite{DonahueJVHZTD14} &64.96 \\\hline
\hline
Euclid             &46.85                 \\\hline
LSVM               &61.44             \\\hline
LMNN               &51.04             \\\hline
MsML                &65.84                  \\\hline
MsML+               &66.61               \\\hline
MsML+*              &67.86             \\\hline
\end{tabular}
\end{table}

\subsection{Birds-2011}

{\it birds11} is the Caltech-USCD-200-2011 birds dataset for bird species~\cite{WahCUB2011}. There are 200 classes with $11,788$ images and each class has roughly 30 images for training. We use the version with ground truth bounding box. Table~\ref{rlr3} compares the proposed method to the state-of-the-art baselines. First, it is obvious that the performance of MsML is significantly better than all baseline methods as the observation above. Second, although Symb~\cite{ChaiLZ13} combines segmentation and localization, MsML outperforms it by $9\%$ without any time consuming operator. Third, Symb* and Ali* mirror the training images to improve their performances, while MsML is even better than them without this trick. Finally, MsML outperforms the method combining DeCAF features and DPD models~\cite{ZhangFID13}, which is due to the fact that most of studies for FGVC ignore choosing the appropriate base classifier and simply adopt linear SVM with the one-vs-all strategy. For comparison, we also report the result mirroring training images which is denoted as {\bf MsML+*}. It provides another $1\%$ improvement over MsML+ as shown in Table~\ref{rlr3}.

To illustrate the capacity of MsML in exploring the correlation among classes, which makes it more effective than a simple one-vs-all classifier for FGVC, we conduct one additional experiment. We randomly select 50 classes from {\it birds11} as the {\it target classes} and use the test images from the target classes for evaluation. When learning the metric, besides the training images from $50$ target classes, we sample $k$ classes from $150$ unselected ones as the {\it auxiliary} classes, and use training images from the auxiliary classes as additional training examples for DML. Fig.~\ref{fig:5} compares the performance of LSVM and MsML with the increasing number of auxiliary classes. It is not surprising to observe that the performance of LSVM decreases a little since it is unable to explore the supervision information in the auxiliary classes to improve the classification accuracy of target classes and more auxiliary classes just intensify the class imbalance problem. In contrast, the performance of MsML improves significantly with increasing auxiliary classes, indicating that MsML is capable of effectively exploring the training data from the auxiliary classes and therefore is particularly suitable for FGVC.

\subsection{Stanford Dogs}

{\it S-dogs} is the Stanford dog species dataset~\cite{khosla2011novel}. It contains $120$ classes and $20,580$ images, where 100 images from each class is used for training. Since it is the subset of ImageNet~\cite{ILSVRCarxiv14}, where DeCAF model is trained from, we just report the result in Table~\ref{rlr4} as reference.

\begin{table}[!ht]
\centering
\caption{Comparison of mean accuracy($\%$) on {\it S-dogs} dataset. ``*'' denotes the method that mirrors training images.}
\label{rlr4}
\begin{tabular}{|l||c|}
\hline
Methods &    Mean Accuracy ($\%$) \\\hline
SIFT~\cite{khosla2011novel}         &22.00\\\hline
Edge Templates~\cite{YangBWS12}     &38.00\\\hline
Symb~\cite{ChaiLZ13}     & 44.10    \\\hline
Symb*~\cite{ChaiLZ13}    & 45.60    \\\hline
Ali*~\cite{GavvesFSST13} & 50.10   \\\hline
\hline
Euclid             &59.22\\\hline
LSVM               &65.00\\\hline
LMNN               &62.17  \\\hline
MsML               &69.07  \\\hline
MsML+              &69.80      \\\hline
MsML+*             &70.31   \\\hline
\end{tabular}
\end{table}

\subsection{Comparison of Efficiency}

In this section, we compare the training time of the proposed algorithm for high dimensional DML to that of LSVM and LMNN. MsML is implemented by Julia, which is a little slower than C{\footnote{Detailed comparison can be found in http://julialang.org}}, while LSVM uses the LIBLINEAR package, the state-of-the-art algorithm for solving linear SVM implemented mostly in C. The core part of LMNN is also implemented in C. The time for feature extraction is not included here because it is shared by all the methods in comparison. The running time for MsML includes all operational cost (i.e., the cost for sampling triplet constraints, computing random projections and low rank approximation).

\begin{table}[!ht]
\centering
\caption{Comparison of running time (seconds).}
\label{rlr5}
\begin{tabular}{|l||p{1.4cm}|p{1.4cm}|p{1cm}|p{1cm}|}
\hline
Methods &{\it cats\&dogs}& {\it 102flowers} &{\it birds11} &{\it S-dogs}\\\hline
LSVM    &196.2 & 309.8            &1,417.0           &1,724.8             \\\hline
LMNN    &832.6& 702.7            &1,178.2           &1,643.6             \\\hline
MsML     &164.9& 174.4            &413.1            &686.3              \\\hline
MsML+    &337.2& 383.7            &791.3            &1,229.7             \\\hline
\end{tabular}
\end{table}

Table~\ref{rlr5} summarizes the results of the comparison. First, it takes MsML about $1/3$ of the time to complete the computation compared to LMNN. This is because MsML employs a stochastic optimization method to find the optimal distance metric while LMNN is a batch learning method. Second, we observe that the proposed method is significantly more efficient than LSVM on most of datasets. The high computational cost of LSVM mostly comes from two aspects. First, LSVM has to train one classification model for each class, and becomes significantly slower when the number of classes is large. Second, the fact that images from different classes are visually similar makes it computationally difficult to find the optimal linear classifier that can separate images of one class from images from the other classes. In contrast, the training time of MsML is independently from the number of classes, making it more appropriate for FGVC. Finally, the running time of MsML+ with $134,016$ features only doubles that of MsML, which verifies that the proposed method is linear in dimensionality ($\Ocal(d)$).

\section{Conclusion}
\label{sec:conclude}

In this paper, we propose a multi-stage metric learning framework for high dimensional FGVC problem, which addresses the challenges arising from high dimensional DML. More specifically, it divides the original problem into multiple stages to handle the challenge arising from too many triplet constraints, extends the theory of dual random projection to address the computational challenge for high dimensional data, and develops a randomized low rank matrix approximation algorithm for the storage challenge. The empirical study shows that the proposed method with general purpose features yields the performance that is significantly better than the state-of-the-art approaches for FGVC. In the future, we plan to combine the proposed DML algorithm with segmentation and localization to further improve the performance of FGVC. Additionally, since the proposed method is a general DML approach, we will try to apply it for other applications with high dimensional features.

\paragraph{Acknowledgments} Qi Qian and Rong Jin are supported in part by ARO (W911NF-11-1-0383), NSF (IIS-1251031) and ONR (N000141410631).

{\small
\bibliographystyle{ieee}
\bibliography{nmetric}
}

\end{document}